%% file: neurips_paper.tex
\documentclass{article}

\usepackage[preprint]{neurips_2026}

\usepackage[utf8]{inputenc}
\usepackage[T1]{fontenc}
\usepackage{hyperref}
\usepackage{url}

\usepackage{booktabs}
\usepackage{microtype}
\usepackage{graphicx}
\usepackage{subcaption}
\usepackage{pifont}
\usepackage{enumitem}
\usepackage{amsmath}
\usepackage{amssymb}
\usepackage{mathtools}
\usepackage{amsthm}
\usepackage{algorithm}
\usepackage[noend]{algpseudocode}
\usepackage[most]{tcolorbox}
\tcbuselibrary{breakable, listings}
\usepackage{colortbl}
\usepackage{float}
\usepackage[capitalize,noabbrev]{cleveref}

\theoremstyle{plain}
\newtheorem{theorem}{Theorem}[section]
\newtheorem{proposition}[theorem]{Proposition}
\newtheorem{lemma}[theorem]{Lemma}

\theoremstyle{definition}

\newtheorem{assumption}[theorem]{Assumption}
\theoremstyle{remark}

\definecolor{studentbg}{HTML}{f8cecc}
\definecolor{teacherbg}{HTML}{d5e8d4}

\title{PAINT: Partial-Solution Adaptive\\ Interpolated Training for Self-Distilled Reasoners}

\author{%
  Zhiquan Tan$^{1}$\thanks{Equal contribution.} \quad
  Yinrong Hong$^{2}$\footnotemark[1] \\
  $^{1}$Tsinghua University \quad $^{2}$Beihang University
}
\begin{document}

\maketitle
\setcounter{footnote}{0}

\begin{abstract}

Improving large language model (LLM) reasoning requires supervision that is both aligned with the model's own test-time states and informative at the token level. Reinforcement learning with verifiable rewards provides on-policy exploration but offers sparse, high-variance credit; supervised fine-tuning and distillation provide dense targets but often train on fixed trajectories or rely on stronger teachers. Recent privileged on-policy self-distillation explores a middle ground by scoring student rollouts with the same model under verified solution context~\citep{zhao2026self}. We revisit this setting through a contextual re-scoring lens: for reasoning, the important choices are not only whether privileged context is available, but how much of it should be revealed and where its distribution should shape the student. We propose \emph{PAINT} (\emph{Partial-solution Adaptive INterpolated Training}), which masks the verified solution according to rollout-reference overlap and applies a small energy-space interpolation on a sparse set of entropy-mismatch token positions. Across competition-level math benchmarks, PAINT consistently improves over a strong prior on-policy self-distillation baseline at all three Qwen3 scales. On Qwen3-8B, it raises macro Avg@12 by 2.1 points over this prior baseline and 2.9 points over GRPO\footnote{Our implementation can be found in \url{https://github.com/tzq1999/PAINT}}.

\end{abstract}

\input{sections/intro}
\input{sections/preliminary}
\input{sections/method}

\input{sections/experiments}
\input{sections/related}
\input{sections/conclusion}

\clearpage
\bibliographystyle{plainnat}
\bibliography{cite}

\clearpage
\onecolumn
\input{sections/appendix}

\end{document}

%% file: sections/intro.tex
\section{Introduction}

Reasoning post-training increasingly determines whether LLMs can solve difficult mathematical and scientific tasks, a capability made visible by chain-of-thought prompting and math-centered evaluation~\citep{wei2022chain,cobbe2021training,hendrycks2021measuring}. The central challenge is not just to reward correct final answers, but to provide supervision that is both \emph{deployment-aligned} and \emph{informative}: training should act on states the model actually visits at inference time, while still giving more guidance than a single sequence-level correctness bit.

Current recipes occupy different points in this trade-off. Reinforcement learning with verifiable rewards (RLVR), including GRPO-style methods~\citep{shao2024deepseekmath,guo2025deepseek,yu2025dapo}, optimizes the desired outcome on on-policy samples, but its reward is sparse and often costly to estimate. Supervised fine-tuning (SFT) on curated reasoning traces~\citep{guha2025openthoughtsdatarecipesreasoning,xiao2026mimov2flashtechnicalreport} is cheap and dense, but trains on reference prefixes rather than the model's own prefixes. Distillation supplies token-level targets, yet standard forms either remain off-policy or depend on an external stronger model~\citep{hinton2015distillingknowledgeneuralnetwork,agarwal2024policy,lu2025onpolicydistillation}.

Prior privileged on-policy self-distillation establishes a strong baseline in this middle ground~\citep{zhao2026self}. Instead of asking a separate teacher to supervise the student, it uses verified solutions as training-time context for another view of the same model, then scores the student's own rollouts under that privileged view. This gives dense on-policy supervision without introducing a stronger model. However, for reasoning tasks, the recipe leaves two important questions under-specified. A fully revealed solution can make the privileged distribution overly sharp and tied to one trace, while matching that distribution at every generated token can spend substantial loss on stylistic continuation rather than mathematical decisions.

\begin{figure*}[t]
\centering
\makebox[\textwidth][c]{%
  \includegraphics[width=1.1\textwidth]{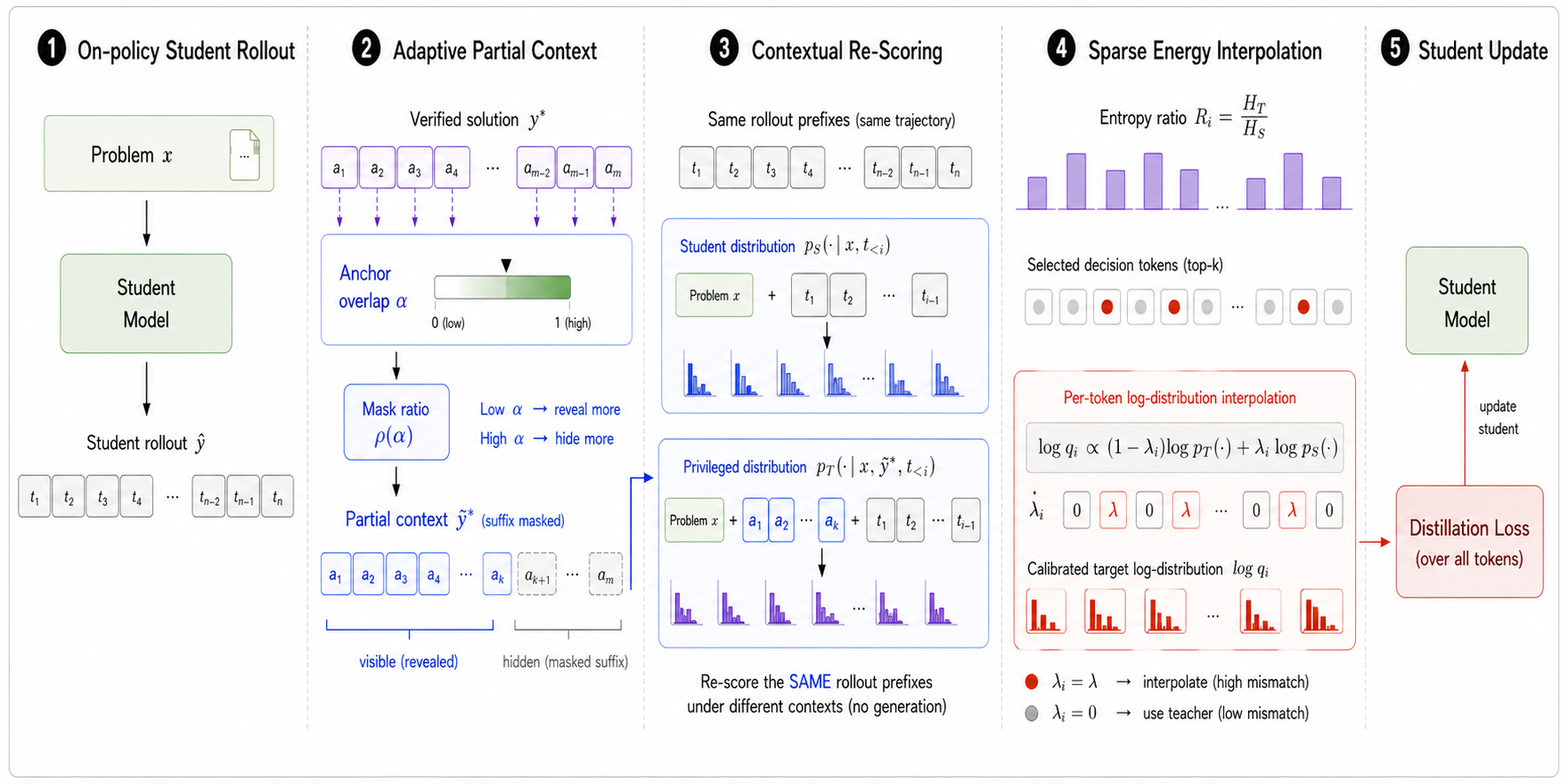}%
}
\captionsetup{justification=justified,width=0.95\textwidth,singlelinecheck=false}
\caption{\textbf{PAINT training pipeline.} PAINT samples an on-policy rollout, uses rollout-reference overlap $\alpha$ to form a suffix-masked solution $\tilde{y}^{\star}$, re-scores the same prefixes with a fixed privileged view, and applies small energy interpolation only on entropy-mismatch positions.}
\label{fig:mainfig}
\end{figure*}

Our theoretical lens makes this gap precise. Building on the idea that richer training-time context can be distilled into deployment-time behavior~\citep{snell2022learning,huang2022context,qicontext}, we view the privileged branch as the same autoregressive model under a context-induced energy landscape over continuations. On a fixed student prefix, adding solution context changes the local next-token energy rather than the visited state itself. This re-scoring view turns the open design choices into two principles. First, context exposure controls a strength-diversity trade-off: more revealed solution context gives sharper corrective guidance, but can collapse the target toward one reference trace. Second, privileged reweighting is local: only some prefixes expose a useful mismatch between the student and the privileged scorer.

As summarized in \Cref{fig:mainfig}, PAINT instantiates these principles with two simple controls. For context exposure, \textbf{P}artial-solution \textbf{A}daptive \textbf{IN}terpolated \textbf{T}raining uses rollout-reference overlap as a proxy for alignment: rollouts with low overlap receive more visible solution context, while already aligned rollouts hide more of the reference to encourage broader reasoning. For local calibration, PAINT applies sparse teacher energy interpolation, ranking positions by a teacher/student entropy-ratio statistic and gently interpolating the privileged target only on selected entropy-mismatch positions before the standard distillation objective.

Empirically, PAINT turns the prior on-policy self-distillation framework from a strong baseline into a stronger reasoning recipe. Across Qwen3-8B, 4B, and 1.7B, PAINT consistently improves over the vanilla prior self-distillation baseline, with macro Avg@12 gains between 0.8 and 2.1 points. It also matches or exceeds GRPO across scales while using short single-rollout training rather than grouped long rollouts, yielding substantially better rollout-token efficiency in our setting.

Our main contributions are summarized below:
\begin{itemize}[leftmargin=*, itemsep=0pt, parsep=0pt]
  \item We develop a contextual re-scoring view of privileged on-policy distillation and identify two reasoning-specific bottlenecks: how much solution context to reveal, and which token positions should receive calibrated teacher targets.
  \item We propose PAINT, which combines overlap-adaptive solution masking with sparse teacher energy interpolation, directly instantiating the two design principles suggested by the re-scoring view.
  \item We evaluate PAINT on three competition-level math benchmarks and three model scales, showing consistent gains over the prior on-policy self-distillation baseline and SFT, plus GRPO-level or better accuracy with a much smaller rollout budget.
\end{itemize}

%% file: sections/preliminary.tex
\section{Background}
\label{sec:opsd_background}

Reasoning post-training has a basic tension between \emph{where the trajectory comes from} and \emph{how informative the supervision is}. SFT and off-policy distillation provide dense token-level targets, but train on reference prefixes rather than the model's own prefixes. RLVR samples on-policy trajectories, but its correctness reward is sparse and sequence-level. On-policy distillation combines on-policy data with dense targets~\citep{agarwal2024policy, lu2025onpolicydistillation, xuspeculative}, but usually assumes an external teacher. The prior OPSD framework~\citep{zhao2026self} resolves this tension by using the verified solution to construct a privileged training-time view of the same model, thereby combining on-policy rollouts with dense token-level supervision without introducing a separate teacher model.

For each training example $(x,y^\star)\in\mathcal{S}$, let $z$ denote a privileged context derived from the verified solution. This framework defines two views:
\[
q_{\mathrm{base}}(\cdot \mid x) \;\triangleq\; p_\theta(\cdot \mid x),
\qquad
q_{\mathrm{aug}}(\cdot \mid x, z) \;\triangleq\; p_\phi(\cdot \mid x, z).
\]
The base view is the deployment-time policy; the augmented view sees the same problem plus privileged solution context and is used only during training. In this background section, we take $z=y^\star$ to describe the full-solution version. The concrete method in \S\ref{sec:method} instead uses the masked trace $z=\tilde{y}^\star$ and fixes $\phi=\theta_0$.

The baseline samples $\hat{y}=(\hat{y}_1,\ldots,\hat{y}_{|\hat{y}|})\sim q_{\mathrm{base}}(\cdot \mid x)$ and evaluates both views on the same sampled prefixes. In the full-solution background form used to define the baseline, define
\[
p_T^n(\cdot)
\triangleq
q_{\mathrm{aug}}\!\left(\cdot \mid x, y^\star, \hat{y}_{<n}\right),
\qquad
p_S^n(\cdot)
\triangleq
q_{\mathrm{base}}\!\left(\cdot \mid x, \hat{y}_{<n}\right).
\]
The rollout discrepancy is then
\begin{equation}
\begin{aligned}
\mathcal{D}_{\mathrm{roll}}(\hat{y}; x, y^\star)
&\triangleq \frac{1}{|\hat{y}|}
\sum_{n=1}^{|\hat{y}|}
D\Bigl(
p_T^n
\,\big\|\,
p_S^n
\Bigr),
\end{aligned}
\label{eq:per-token}
\end{equation}
where $D$ may be instantiated as forward KL, reverse KL, or a Jensen-Shannon-type divergence. The full objective is
\begin{equation}
\mathcal{J}_{\mathrm{OPSD}}(\theta)
=
\mathbb{E}_{(x,y^\star) \sim \mathcal{S}}
\mathbb{E}_{\hat{y} \sim q_{\mathrm{base}}(\cdot \mid x)}
\Bigl[
\mathcal{D}_{\mathrm{roll}}(\hat{y}; x, y^\star)
\Bigr].
\label{eq:opsd}
\end{equation}
The essential property is that supervision is dense and token-level, while the visited states still come from the policy being improved.

Our main implementation uses full-vocabulary forward KL with pointwise clipping. Given teacher targets $p_T^n$ and student distributions $p_S^n$ on the same sampled prefixes, define
\begin{equation}
D_{\mathrm{fKL}}^{\mathrm{clip}}(p_T,p_S;\hat{y})
=
\frac{1}{|\hat{y}|}
\sum_{n=1}^{|\hat{y}|}
\sum_{v \in \mathcal{V}}
\min\!\left(
p_T^n(v)
\bigl[
\log p_T^n(v)-\log p_S^n(v)
\bigr],
\tau
\right),
\label{eq:clipped_fkl}
\end{equation}
where $\tau$ is a pointwise clipping threshold. 

%% file: sections/method.tex
\section{Method}
\label{sec:method}

Our method modifies how a prior on-policy self-distillation scaffold~\citep{zhao2026self} constructs and uses its privileged teacher. We begin with a contextual re-scoring lens that separates the fixed scaffold--student rollouts scored under privileged context--from the design choices left to the practitioner. This lens points to two fragile parts of reasoning training: how much solution context to reveal, and which token positions should receive calibrated teacher targets. The next three subsections develop these choices in order: contextual re-scoring, adaptive partial-solution conditioning, and sparse interpolation.

\subsection{A Theoretical Lens: Contextual Re-Scoring}
\label{sec:method_instantiation}

The key modeling choice is not where the trajectory comes from, but how a student-produced prefix is scored. Once the deployment policy samples a rollout $\hat{y}$, each prefix $\hat{y}_{<n}$ can be evaluated under two contexts: the ordinary deployment prompt, or a privileged prompt that also contains solution information. The augmented branch is therefore a scorer over the student's visited states, not a generator of an alternative trajectory. If privileged context reweights preferences on the same generation tree, the main design questions become how much evidence to reveal to that scorer and how strongly its reweighted target should influence each position.

\begin{figure*}[t]
\centering
\begin{tcolorbox}[
    enhanced,
    colback=studentbg,
    colframe=red!45!black,
    coltitle=red!45!black,
    colbacktitle=studentbg,
    title=\textbf{Deployment prompt: student view},
    fonttitle=\bfseries,
    rounded corners,
    boxrule=0.8pt,
    width=0.94\textwidth
]
\ttfamily\small
Problem: How many ordered pairs of positive integers $(a,b)$ satisfy
$\frac{1}{a}+\frac{1}{b}=\frac{1}{6}$?\\[4pt]
Please reason step by step, and put your final answer within \textbackslash boxed\{\}.
\end{tcolorbox}
\vspace{6pt}
\begin{tcolorbox}[
    enhanced,
    colback=teacherbg,
    colframe=green!45!black,
    coltitle=green!45!black,
    colbacktitle=teacherbg,
    title=\textbf{Privileged scoring prompt: fixed teacher view},
    fonttitle=\bfseries,
    rounded corners,
    boxrule=0.8pt,
    width=0.94\textwidth
]
\ttfamily\small
Problem: How many ordered pairs of positive integers $(a,b)$ satisfy
$\frac{1}{a}+\frac{1}{b}=\frac{1}{6}$?\\[4pt]
Here is a part of reference solution to this problem:\\
=== Reference Solution Begin ===\\
Rewrite as $ab=6a+6b$, hence $(a-6)(b-6)=36$. ...(omitted)...\\
=== Reference Solution End ===\\[4pt]
\textbf{After reading the reference solution above, derive the same final answer in your own words.}\\[4pt]
Please reason step by step, and put your final answer within \textbackslash boxed\{\}.\\
\textnormal{[student rollout prefix being scored]}
\end{tcolorbox}
\caption{\textbf{Prompt views for privileged re-scoring.} The fixed teacher sees a masked reference solution and scores the same student-produced prefixes, yielding token-level targets without generating a separate trajectory.}
\label{fig:prompts}
\end{figure*}

To make this re-scoring view precise, we treat the privileged branch as the same autoregressive model under a different context-induced energy landscape over continuations. The following proposition records the standard correspondence between autoregressive likelihoods and Gibbs distributions on the generation tree.

\begin{proposition}[Context-induced Gibbs teacher]
\label{prop:method_arm_gibbs}
Fix a prompt context $z$, a finite horizon $T$, and a finite vocabulary. If the frozen autoregressive model is strictly positive, then the trajectory energy
\[
E_z(\tau)
=
-\sum_{i=1}^T \log p_{\theta_0}(\tau_i\mid z,\tau_{<i})
\]
has a Gibbs distribution equal to the model's autoregressive distribution over continuations:
\[
\frac{\exp(-E_z(\tau))}{\sum_{\tau'}\exp(-E_z(\tau'))}
=
p_{\theta_0}(\tau\mid z).
\]
Conversely, any additive energy $E_z(\tau)=\sum_{i=1}^T c_z(s_i,a_i)$ on the finite generation tree induces the normalized next-token policy
\[
p_E(a_i\mid z,a_{<i})
=
\exp\!\left(-c_z(s_i,a_i)+V_z(s_{i+1})-V_z(s_i)\right),
\]
where $s_i=(z,a_{<i})$ and $V_z$ is the log-partition soft value over suffix continuations.
\end{proposition}

\noindent The proof is given in \Cref{prop:arm_gibbs_correspondence}. Intuitively, extra solution context does not create a different reasoning engine; it tilts the same model toward continuations that are more compatible with the reference solution. This suggests two concrete intervention points. First, because richer context can make the teacher target overly sharp, we control how much of the reference solution is visible. Second, because the reweighted target is not equally useful on every generated token, we calibrate it only on a sparse set of positions. These two intervention points motivate adaptive masking and sparse interpolation.

Although the abstract objective permits both views to share the current parameters, our main instantiation uses a fixed privileged teacher with $\phi=\theta_0$ and updates only the student parameters $\theta$. Under the re-scoring view, fixing the teacher keeps the privileged energy landscape stable and avoids having the target drift toward the student during training.

\subsection{Overlap-Adaptive Partial Solution Conditioning}
\label{sec:method_masking}

Once the privileged teacher is fixed, the next question is how informative it should be. Full-solution conditioning gives a low-entropy target that can over-commit to one trace, whereas too little context makes the teacher too similar to the student. We therefore reveal more when the student is far from the reference structure and hide more when it is already aligned. This design has two parts: how much of the solution to hide, and which part to hide.

We first address the amount of masking. In a stylized view, increasing the hidden ratio increases the entropy of the teacher target, which encourages broader generalization, but it also makes the teacher harder for the student to fit, increasing optimization noise. The next theorem formalizes this trade-off and gives the monotonic rule used by PAINT: the preferred hidden ratio decreases as alignment error grows. Let $\mathcal{X}$ be the finite set of valid continuations and let $\rho\in[0,1)$ denote a continuous hidden ratio for a contiguous masked block. Let $q_0\in\Delta(\mathcal{X})$ be the low-entropy teacher target with no masking, and let $q_1\in\Delta(\mathcal{X})$ be the more diffuse target induced by maximal contiguous masking. We interpolate them by
\[
q_\rho(\tau)
=
(1-\rho)q_0(\tau)+\rho q_1(\tau).
\]
For student parameters $\theta$, define the uncentered gradient second moment
\[
M_2(q,\theta)
=
\mathbb{E}_{\tau\sim q}
\left[
\|\nabla_\theta \log p_\theta(\tau)\|^2
\right],
\]
and let
\[
v_0
=
M_2(q_0,\theta),
\qquad
\epsilon(\theta)
=
M_2(q_1,\theta)-M_2(q_0,\theta).
\]
Here $\epsilon$ measures the extra optimization noise induced by aggressive masking, while the entropy of $q_\rho$ measures how much multi-modality the teacher exposes to the student. By linearity of expectation, the second moment along the interpolation path is
\[
M_2(q_\rho,\theta)
=
v_0+\rho\epsilon.
\]

\begin{theorem}[Monotonicity of the optimal adaptive mask ratio]
\label{thm:method_adaptive_mask_monotonicity}
Assume $q_\rho(\tau)>0$ on the relevant support and define the Shannon entropy
\[
H(\rho)
=
-\sum_{\tau\in\mathcal{X}} q_\rho(\tau)\log q_\rho(\tau).
\]
Assume $q_0\neq q_1$. For coefficients $\beta,\lambda>0$, consider the objective
\[
\mathcal{L}(\rho,\epsilon)
=
\beta H(\rho)-\lambda(v_0+\rho\epsilon).
\]
Whenever an interior maximizer $\rho^\star(\epsilon)\in(0,1)$ exists, its differentiable branch is a strictly decreasing function of the alignment-error proxy:
\[
\frac{d \rho^\star}{d\epsilon}<0.
\]
\end{theorem}

A detailed proof is given in \Cref{thm:appendix_adaptive_mask_monotonicity}. The theorem gives the qualitative rule we need: when the student is poorly aligned with the privileged target (large $\epsilon$), the teacher should reveal more and hide less; as alignment improves, the teacher can hide more and encourage broader generalization. To instantiate this principle on each rollout, we need a practical proxy for inverse alignment error. After sampling $\hat{y}\sim q_{\mathrm{base}}(\cdot \mid x)$, the trainer extracts a set of reference anchors $\mathcal{A}(y^\star)$, including the boxed answer, formula-like spans, and informative numbers, after normalization of common mathematical notation. It then computes a recall-style overlap score
\begin{equation}
\alpha(x,\hat{y},y^\star)
=
\frac{1}{|\mathcal{A}(y^\star)|}
\sum_{a\in\mathcal{A}(y^\star)}
\mathbf{1}\!\left\{a \subseteq \operatorname{Norm}(\hat{y})\right\},
\qquad
\alpha\in[0,1].
\end{equation}
If no anchors are found, the implementation sets $\alpha=0$. The exact final-answer verifier is still computed for logging and analysis, but the teacher schedule is driven by this continuous overlap score rather than by a hard correct/incorrect branch. Higher overlap means that the rollout already reproduces more of the reference's mathematical structure, so it is safe to hide more. We therefore use the affine schedule
\begin{equation}
\rho(\alpha)
=
\rho_{\mathrm{wrong}}
+
\left(\rho_{\mathrm{correct}}-\rho_{\mathrm{wrong}}\right)\alpha,
\qquad
\rho_{\mathrm{correct}}>\rho_{\mathrm{wrong}}.
\end{equation}
Thus $\alpha$ acts as a practical proxy for inverse alignment error, and $\rho(\alpha)$ is the concrete monotone curriculum suggested by the theorem.

The theorem above explains how much privileged information to withhold, but it does not say which part of the solution should be hidden. We restrict attention to contiguous masks because they preserve local coherence. Among them, we prefer suffix masks because early reference tokens usually carry the setup, decomposition, and high-level plan that make later reasoning predictable, whereas later tokens more often execute downstream algebra or state the final answer. Appendix \ref{sec:suffix_masking_theory} makes this intuition formal: masking induces a posterior mixture over hidden completions, which smooths the teacher target and Rao--Blackwellizes the forward-KL gradient relative to conditioning on a single revealed completion. For a fixed masking budget $k$, the preferred mask position is therefore the one that removes the least privileged information relevant to the teacher's next action. Let
\[
M_t = y^\star_{t:t+k-1},
\qquad
C_t = (y^\star_{<t}, y^\star_{t+k:T}),
\]
and quantify this trade-off by
\[
\Gamma(t)
=
\mathbb{E}\!\left[I(A_n;M_t\mid x,\hat{y}_{<n},C_t)\right],
\]
the expected privileged information about the teacher's next action that is available only in the hidden block. In typical reasoning traces, early tokens more often carry definitions, reductions, and the high-level plan, whereas later tokens more often execute downstream algebra or state the final answer. Under this causal information asymmetry, suffix masking is optimal.

\begin{theorem}[Optimality of suffix masking]
\label{thm:method_suffix_masking_optimality}
Assume later mask positions leave weakly more informative visible context for predicting the teacher's next action than earlier mask positions. Then $\Gamma(t)$ is non-increasing in $t$. Hence the minimum expected information loss is attained at the largest feasible mask position,
\[
t^\star=T-k+1,
\]
which masks the suffix of length $k$.
\end{theorem}

\noindent Appendix \ref{sec:suffix_masking_theory} provides the detailed derivation, including entropy smoothing of masked targets, the forward-KL information decomposition, the Rao--Blackwellized variance reduction, the visible-context duality, and the proof of \Cref{thm:method_suffix_masking_optimality}. Combining the monotone ratio schedule with this placement result gives the privileged trace used in practice:
\begin{equation}
\tilde{y}^\star
=
M_{\mathrm{tail}}\!\left(y^\star; \rho(\alpha)\right).
\end{equation}
In the main runs, $M_{\mathrm{tail}}$ removes a character-level suffix of the reference solution and replaces it with an omission marker. The code also supports prefix, middle, and punctuation-aware suffix masking variants, but the default setting used in the reported experiments is suffix masking.

\subsection{Sparse Teacher Energy Interpolation}
\label{sec:method_calibration}

After deciding what the teacher sees, the final question is how literally to trust the resulting token distribution. Even with adaptive masking, some positions contain useful privileged guidance, while others mostly reflect stylistic continuation or residual uncertainty. We therefore do not calibrate the teacher everywhere. Instead, we spend a sparse calibration budget on positions with the largest entropy mismatch, where a small energy-space correction can soften over-specific privileged targets without washing out the signal on every token. Given the masked teacher prompt, define the student and teacher next-token distributions on the same student prefix:
\begin{equation}
p_S^n(\cdot)
\triangleq
q_{\mathrm{base}}(\cdot \mid x, \hat{y}_{<n}),
\qquad
p_T^n(\cdot)
\triangleq
q_{\mathrm{aug}}(\cdot \mid x, \tilde{y}^\star, \hat{y}_{<n}).
\end{equation}
The trainer therefore applies a conservative interpolation step in token-energy space before computing the distillation loss, but only on positions selected by a sparse budget. Let $z_S^n$ and $z_T^n$ be the raw student and teacher logits for prefix $\hat{y}_{<n}$, and let $p_S^n=\operatorname{softmax}(z_S^n/T)$ and $p_T^n=\operatorname{softmax}(z_T^n/T)$ be the temperature-scaled distributions. Define token energies
\[
E_S^n(v)=-\log p_S^n(v),
\qquad
E_T^n(v)=-\log p_T^n(v).
\]

Let
\begin{equation}
H_S^n = - \sum_{v \in \mathcal{V}} p_S^n(v)\log p_S^n(v),
\qquad
H_T^n = - \sum_{v \in \mathcal{V}} p_T^n(v)\log p_T^n(v),
\end{equation}
and define the entropy ratio
\begin{equation}
R_n=\frac{H_T^n}{H_S^n}.
\label{eq:entropy_ratio}
\end{equation}
We use $R_n$ only as a rank statistic. Under the local Gaussian energy approximation and a shared entropy--variance equation of state, \Cref{prop:entropy_ratio_sparse_selection} shows that the signed differential-entropy reduction from applying the fixed interpolation strength at position $n$ can be written as
\begin{equation}
\Delta I_n(k_{\mathrm{base}})
=
\frac{1}{2}\log\!\left(1-k_{\mathrm{base}}+k_{\mathrm{base}}R_n^\gamma\right).
\label{eq:entropy_ratio_information_gain}
\end{equation}
This signed gain is strictly order-preserving in $R_n$ for $k_{\mathrm{base}}\in(0,1)$ and $\gamma>0$. Rather than learning a continuous $k_n$, which can collapse toward replacing the privileged target by the student distribution on large-ratio positions, the implementation fixes $k_{\mathrm{base}}$ and only decides where to intervene. Let $\mathcal{J}$ be the valid generated-token positions and let $S=\lceil f_{\mathrm{top}}|\mathcal{J}|\rceil$. The sparse selector corresponds to the finite $L_0$-budget problem
\begin{equation}
\begin{aligned}
\max_{\mathbf{m}\in\{0,1\}^{|\mathcal{J}|}}
\quad&
\sum_{n\in\mathcal{J}} m_n\,\Delta I_n(k_{\mathrm{base}})
\\
\mathrm{s.t.}\quad&
\sum_{n\in\mathcal{J}} m_n=S .
\end{aligned}
\label{eq:sparse_selector_objective}
\end{equation}
Since $\Delta I_n(k_{\mathrm{base}})$ and $R_n$ induce the same ordering, the global optimum is obtained by choosing the $S$ largest entropy ratios, with arbitrary tie-breaking at the cutoff. For each sample, this top-$f_{\mathrm{top}}$ set is denoted by $\mathcal{I}$. Conceptually, only on this sparse set is the teacher energy moved toward the student energy:
\begin{equation}
\tilde{E}_T^n(v)
=
\begin{cases}
(1-k_{\mathrm{base}})E_T^n(v)+k_{\mathrm{base}}E_S^n(v), & n\in\mathcal{I},\\
E_T^n(v), & n\notin\mathcal{I}.
\end{cases}
\label{eq:sparse_energy_interp}
\end{equation}
The reported runs use a constant $k_{\mathrm{base}}$ and do not use an entropy-ratio scale term; the entropy ratio only determines which positions receive interpolation. Finally,
\begin{equation}
\tilde{p}_T^n(v)
=
\frac{\exp(-\tilde{E}_T^n(v))}
{\sum_{u\in\mathcal{V}}\exp(-\tilde{E}_T^n(u))}.
\label{eq:teacher_interp}
\end{equation}
Equivalently, on selected positions,
\[
\tilde{p}_T^n(v)
\propto
\bigl(p_T^n(v)\bigr)^{1-k_{\mathrm{base}}}
\bigl(p_S^n(v)\bigr)^{k_{\mathrm{base}}},
\]

\Cref{prop:precision_weighted_calibration} gives a broader energy-fusion motivation; the implementation uses the sparse constant-weight version above to avoid washing out the privileged signal across all tokens.

Substituting the interpolated teacher targets into Eq.~\ref{eq:clipped_fkl} gives the rollout loss used in the main experiments:
\begin{equation}
\widetilde{\mathcal{D}_{\mathrm{roll}}}(\hat{y}; x, \tilde{y}^\star)
=
D_{\mathrm{fKL}}^{\mathrm{clip}}
\left(
\{\tilde{p}_T^n\}_{n=1}^{|\hat{y}|},
\{p_S^n\}_{n=1}^{|\hat{y}|};
\hat{y}
\right).
\label{eq:main_loss}
\end{equation}

The complete update, including overlap-adaptive masking and sparse teacher energy interpolation, is summarized in Appendix \ref{sec:appendix_training_algorithm}.

%% file: sections/experiments.tex
\section{Experiments}
\label{sec:experiments}

Our experiments test whether overlap-adaptive partial-solution conditioning and sparse teacher interpolation improve a strong prior on-policy self-distillation baseline under matched data, backbones, rollout budget, and privileged-teacher setup. We compare against supervised imitation, reward-based RL, and the closest prior self-distillation baseline across model scales and competition-level math benchmarks.

\subsection{Experimental Setup}

\textbf{Models and data.} We use the instruct-tuned Qwen3 family~\citep{qwen3technicalreport} at three scales: Qwen3-1.7B, Qwen3-4B, and Qwen3-8B. Training data comes from the mathematical reasoning subset of OpenThoughts~\citep{guha2025openthoughtsdatarecipesreasoning}; for each scale, we sample up to 30K problem-solution pairs with verified final answers and reasoning traces. Evaluation covers AIME 2024, AIME 2025, and HMMT 2025.

\textbf{Compared post-training recipes.} We evaluate four updates from the same base checkpoints. \textbf{SFT} directly imitates the reference reasoning traces and represents an off-policy token-level baseline. \textbf{GRPO}~\citep{shao2024deepseekmath} optimizes binary answer rewards with grouped rollouts; its maximum generation length is 16k. \textbf{OPSD} is the closest prior on-policy self-distillation baseline~\citep{zhao2026self}: it scores each student rollout with a privileged full-solution view but does not use PAINT's adaptive masking or sparse interpolation. \textbf{PAINT} is our method, which keeps the same rollout-and-score structure while changing what the privileged teacher sees and where its distribution is calibrated.

\textbf{Training and evaluation protocol.} Unless otherwise stated, PAINT uses a fixed privileged teacher anchored at the initial checkpoint while the student policy is updated with LoRA~\citep{hu2022lora}. Teacher forward passes disable the trainable adapters. PAINT conditions the teacher on an overlap-adaptive partial reference solution with $\rho_{\mathrm{wrong}}=0.30$ and $\rho_{\mathrm{correct}}=0.40$, then interpolates teacher logits toward student logits with $k_{\mathrm{base}}=0.03$ on the top $f_{\mathrm{top}}=0.03$ fraction of generated-token positions ranked by the teacher/student entropy ratio. Additional training details appear in Appendix~\ref{sec:experimental_details}.

\subsection{Main Results}
\label{sec:main_results}

\begin{table*}[t]
\centering
\caption{Math reasoning Avg@12 after post-training Qwen3 backbones. Evaluation uses thinking mode, temperature $1.0$, top-p $0.95$, and maximum length $38$k. PAINT and OPSD use the same training split, rollout budget, and fixed privileged teacher; PAINT adds only adaptive partial-solution conditioning and sparse teacher interpolation.}
\vspace{1mm}
\label{tab:main_results}
\resizebox{0.78\textwidth}{!}{
\begin{tabular}{@{}lcccc@{}}
\toprule
\textbf{Backbone / post-training recipe} & \textbf{AIME 2024} & \textbf{AIME 2025} & \textbf{HMMT 2025} & \textbf{Macro avg.} \\
\midrule
\multicolumn{5}{@{}l}{\textit{Qwen3-8B Instruct}} \\
\quad No update              & 75.8 & 65.6 & 43.9 & 61.8 \\
\quad + SFT                  & 72.3 & 64.2 & 42.9 & 59.8 \\
\quad + GRPO                 & 76.4 & 68.9 & 46.7 & 64.0 \\
\quad + OPSD (prior)         & 77.8 & 70.8 & 45.8 & 64.8 \\
\rowcolor{gray!20}\quad + PAINT (ours)
& \textbf{79.2} & \textbf{72.2} & \textbf{49.4} & \textbf{66.9} \\
\midrule
\multicolumn{5}{@{}l}{\textit{Qwen3-4B Instruct}} \\
\quad No update              & 74.9 & 66.4 & 42.2 & 61.2 \\
\quad + SFT                  & 70.2 & 62.3 & 43.4 & 58.6 \\
\quad + GRPO                 & 75.6 & 68.1 & 44.4 & 62.7 \\
\quad + OPSD (prior)         & 76.4 & 68.3 & 46.1 & 63.6 \\
\rowcolor{gray!20}\quad + PAINT (ours)
& \textbf{76.9} & \textbf{69.7} & \textbf{46.7} & \textbf{64.4} \\
\midrule
\multicolumn{5}{@{}l}{\textit{Qwen3-1.7B Instruct}} \\
\quad No update              & 51.5 & 36.7 & 23.1 & 37.1 \\
\quad + SFT                  & 48.4 & 36.3 & 22.7 & 35.8 \\
\quad + GRPO                 & 51.1 & 38.3 & 23.7 & 37.7 \\
\quad + OPSD (prior)         & 57.2 & 43.9 & 29.2 & 43.4 \\
\rowcolor{gray!20}\quad + PAINT (ours)
& \textbf{59.7} & \textbf{44.2} & \textbf{29.7} & \textbf{44.5} \\
\bottomrule
\end{tabular}
}
\end{table*}

Table~\ref{tab:main_results} shows that PAINT improves over the prior on-policy self-distillation baseline at every model scale: +2.1 points on Qwen3-8B, +0.8 on Qwen3-4B, and +1.1 on Qwen3-1.7B. Since OPSD already beats SFT by turning reference solutions into dense on-policy supervisory distributions, these gains suggest that controlling privileged context and concentrating calibration on entropy-mismatch positions matters beyond the basic re-scoring scaffold.

PAINT also matches or exceeds GRPO across all three scales with a smaller sampling budget: one 1024-token rollout per problem and best checkpoints within 100 steps, compared with 8 GRPO rollouts of up to 16k tokens over a longer horizon. In this OpenThoughts setting, many GRPO batches have zero within-group reward standard deviation, so the larger rollout-token budget can still yield little learning signal.

Finally, SFT degrades on average at all three scales, likely because concise reference traces shorten test-time reasoning. OPSD avoids this by scoring the student's own prefixes instead of training on reference prefixes. PAINT keeps that deployment-aligned property while reducing over-conditioning on a fully revealed solution and equal weighting of mathematical decisions and low-value stylistic continuations.

Appendix~\ref{sec:ablation_studies} isolates the two PAINT components on Qwen3-1.7B. Moderate adaptive suffix masking outperforms fixed full-solution conditioning and heavier masks, suffix placement beats prefix, middle, and random contiguous masks, and sparse interpolation is strongest when applied to a small top-$f_{\mathrm{top}}$ budget rather than every token. These ablations support the intended mechanisms.

%% file: sections/related.tex
\section{Related Work}
\label{sec:related_work}

Reasoning post-training includes RLVR with verifiable rewards~\citep{shao2024deepseekmath,guo2025deepseek,team2025kimi,yu2025dapo,rastogi2025magistral}, process supervision~\citep{lightman2023let,zhang2025lessons}, and SFT or self-improvement over curated and self-generated traces~\citep{guha2025openthoughtsdatarecipesreasoning,xiao2026mimov2flashtechnicalreport,zelikman2022star,gulcehre2023reinforced,wang2023self,yuan2024self}. These approaches provide useful signals but differ in whether supervision is on-policy, dense, or externally annotated. RLVR directly optimizes final-answer correctness but can be statistically inefficient when rewards are sparse; process supervision gives finer feedback but requires additional labels or learned reward models; and imitation-style training depends strongly on the quality and length of reference traces. PAINT targets the complementary regime where verified answers are available but step labels and stronger external teachers are not. In this regime, the main question is how to turn a verified solution into useful token-level feedback without moving training off the student's own trajectories.

PAINT is closest to distillation and privileged-context self-training. Classical and sequence-level distillation transfer teacher distributions or generated sequences~\citep{hinton2015distillingknowledgeneuralnetwork,kim2016sequence,sanh2019distilbert}, while on-policy distillation supervises student-sampled trajectories~\citep{agarwal2024policy,gu2024minillm,xuspeculative,lu2025onpolicydistillation}. Context distillation compresses richer-context behavior into a deployment-time model~\citep{snell2022learning,huang2022context,qicontext}, and privileged on-policy self-distillation applies this idea to verified reasoning solutions~\citep{zhao2026self,hubotter2026reinforcement,shenfeld2026selfdistillationenablescontinuallearning}. PAINT keeps the same rollout re-scoring scaffold but replaces full-solution conditioning and uniform token matching with overlap-adaptive suffix masking and sparse energy interpolation. Unlike external-teacher distillation, the privileged scorer is not a stronger model; the benefit comes from changing the information available at training time and deciding where that information should affect the student. This makes the contribution complementary to reward design and data curation: the training examples and verifier stay fixed, while the privileged scoring distribution is made less over-specific and more selectively applied.

%% file: sections/conclusion.tex
\section{Conclusion}

Reasoning post-training needs feedback that is both dense and aligned with the model's own visited states. We argue that privileged on-policy self-distillation works through contextual re-scoring: solution context reweights the same student prefixes, making context exposure and token-level calibration central design choices. PAINT instantiates this view with overlap-adaptive partial-solution conditioning and sparse teacher energy interpolation. Across mathematical reasoning benchmarks, it improves over the vanilla prior baseline and matches or exceeds GRPO with a much smaller rollout budget; on Qwen3-8B, it gains 2.1 points over OPSD and 2.9 over GRPO. Extending this re-scoring principle beyond verified math traces remains an important next step, especially for domains where partial references are noisier and token-level teacher uncertainty is harder to interpret.

%% file: sections/appendix.tex
\appendix

\section{Training Algorithm}
\label{sec:appendix_training_algorithm}

Algorithm~\ref{alg:adaptive_opsd} summarizes the complete PAINT recipe. Each example contributes one on-policy rollout, one overlap-adaptive masked privileged prompt, one pair of student/teacher distributions per prefix, and one clipped forward-KL loss. Unless otherwise stated, all experiments use this fixed-teacher, suffix-masked, sparsely energy-interpolated full-vocabulary objective. For readability, the algorithm writes the interpolation in its equivalent normalized energy form; the implementation performs the same operation by interpolating logits before the final softmax.

\begin{algorithm*}[t]
\caption{PAINT: complete adaptive privileged re-scoring recipe used in our main experiments}
\label{alg:adaptive_opsd}
\begin{algorithmic}[1]
\Require Reasoning dataset $\mathcal{S} = \{(x_i, y_i^\star)\}_{i=1}^N$, initialized model $p_{\theta_0}$, student policy $q_{\mathrm{base}}$, fixed privileged teacher $q_{\mathrm{aug}}$, anchor extractor $\mathcal{A}$, mask ratios $\rho_{\mathrm{correct}} > \rho_{\mathrm{wrong}}$, sparse interpolation parameters $(k_{\mathrm{base}}, f_{\mathrm{top}})$, clipping threshold $\tau$
\While{not converged}
    \State Sample a minibatch $\mathcal{B} \subset \mathcal{S}$
    \ForAll{$(x, y^\star) \in \mathcal{B}$}
        \State Sample one student rollout $\hat{y} \sim q_{\mathrm{base}}(\cdot \mid x)$
        \State Compute anchor-overlap score $\alpha \gets |\{a\in\mathcal{A}(y^\star):a\subseteq\operatorname{Norm}(\hat{y})\}|/|\mathcal{A}(y^\star)|$
        \State Set mask ratio $\rho \gets \rho_{\mathrm{wrong}}+(\rho_{\mathrm{correct}}-\rho_{\mathrm{wrong}})\alpha$
        \State Build masked privileged trace $\tilde{y}^\star \gets M_{\mathrm{tail}}(y^\star; \rho)$
        \For{$n = 1, \dots, |\hat{y}|$}
            \State $p_S^n \gets q_{\mathrm{base}}(\cdot \mid x, \hat{y}_{<n})$ and $p_T^n \gets q_{\mathrm{aug}}(\cdot \mid x, \tilde{y}^\star, \hat{y}_{<n})$
            \State $E_S^n \gets -\log p_S^n$ and $E_T^n \gets -\log p_T^n$
            \State Compute $H_S^n$, $H_T^n$, and $R_n \gets H_T^n/H_S^n$
        \EndFor
        \State Let $\mathcal{I}$ be the top $f_{\mathrm{top}}$ fraction of generated-token positions by $R_n$
        \State Set $\tilde{E}_T^n\gets(1-k_{\mathrm{base}})E_T^n+k_{\mathrm{base}}E_S^n$ for $n\in\mathcal{I}$, else $\tilde{E}_T^n\gets E_T^n$
        \State $\tilde{p}_T^n(v)\gets\exp(-\tilde{E}_T^n(v))/\sum_{u\in\mathcal{V}}\exp(-\tilde{E}_T^n(u))$
        \State $\ell(x,y^\star) \gets \widetilde{\mathcal{D}_{\mathrm{roll}}}(\hat{y}; x, \tilde{y}^\star)$
    \EndFor
    \State Update the student parameters underlying $q_{\mathrm{base}}$ using $\frac{1}{|\mathcal{B}|}\sum_{(x,y^\star)\in\mathcal{B}}\ell(x,y^\star)$
\EndWhile
\end{algorithmic}
\end{algorithm*}

The trainer also includes optional variants such as EMA teachers, sampled-token objectives, and reason-first teachers, but they are not part of the reported main recipe.

\section{Experimental Details}
\label{sec:experimental_details}

Unless otherwise stated, all privileged on-policy distillation runs use the same post-training setting. We train LoRA adapters with learning rate $5\times10^{-6}$, effective batch size $32$, rank $r=64$, $\alpha=128$, and target modules q\_proj, k\_proj, v\_proj, o\_proj, gate\_proj, up\_proj, and down\_proj. Each update samples one base-policy completion per prompt with sampling temperature $1.1$ and maximum completion length $1024$, and scores the same prefixes with the fixed privileged teacher. For PAINT, we apply suffix masking with $\rho_{\mathrm{wrong}}=0.30$ and $\rho_{\mathrm{correct}}=0.40$, use sparse teacher-energy interpolation with $k_{\mathrm{base}}=0.03$ and $f_{\mathrm{top}}=0.03$, and clip the pointwise forward-KL term at $\tau=0.06$. The vanilla prior self-distillation baseline removes these PAINT-specific adaptive masking and sparse interpolation components while keeping the same training budget and teacher anchoring. We train for $100$ steps and evaluate checkpoints every $20$ steps. Evaluation follows the Qwen3 recommended sampling setup with thinking mode enabled, temperature $1.0$, top-p $0.95$, top-k $-1$, min-p $0.0$, no presence penalty, maximum generation length $38{,}912$, and $12$ samples per prompt.

\section{Ablation Studies}
\label{sec:ablation_studies}

All ablations in this section use Qwen3-1.7B and report AIME 2024 Avg@12. The first two sweeps isolate the masking design with sparse interpolation disabled; the final sweep fixes the best masking setting and ablates sparse interpolation.

\subsection{Mask-Ratio Selection}
\label{sec:mask_ratio_ablation}

The adaptive masking schedule uses two endpoints: $\rho_{\mathrm{wrong}}$, applied when the rollout has little overlap with the verified solution, and $\rho_{\mathrm{correct}}$, applied when the rollout already matches more reference anchors. We report the sweep as $(\rho_{\mathrm{correct}}, \rho_{\mathrm{wrong}})$, with the actual per-example ratio computed as $\rho(\alpha)=\rho_{\mathrm{wrong}}+(\rho_{\mathrm{correct}}-\rho_{\mathrm{wrong}})\alpha$. Table~\ref{tab:mask_ratio_ablation} evaluates this choice with suffix masking and no sparse interpolation.

\begin{table}[htbp]
\centering
\caption{Ablation of the adaptive suffix-mask ratios with suffix masking and sparse interpolation disabled. Rows with equal endpoints correspond to a fixed mask ratio, while unequal endpoints use the overlap-adaptive schedule.}
\label{tab:mask_ratio_ablation}
\vspace{1mm}
\begin{tabular}{@{}ccc@{}}
\toprule
$\boldsymbol{\rho_{\mathrm{correct}}}$ & $\boldsymbol{\rho_{\mathrm{wrong}}}$ & \textbf{AIME 2024} \\
\midrule
0.00 & 0.00 & 57.2 \\
0.25 & 0.25 & 57.8 \\
0.50 & 0.50 & 56.9 \\
0.75 & 0.75 & 55.8 \\
0.50 & 0.25 & 58.1 \\
0.50 & 0.30 & 57.8 \\
\rowcolor{gray!20}
0.40 & 0.30 & \textbf{58.6} \\
\bottomrule
\end{tabular}
\end{table}

The sweep supports the conservative adaptive schedule used in the main experiments. Revealing the full solution is strong but leaves little context control, while a large fixed hidden ratio can remove too much privileged information. Moderate adaptive masking performs best in this sweep: low-overlap rollouts still receive enough visible solution context for correction, whereas high-overlap rollouts hide slightly more of the suffix to avoid over-committing to a single reference trace.

\subsection{Mask Placement Strategy}
\label{sec:mask_strategy_ablation}

We next fix the adaptive endpoints to the main setting, $(\rho_{\mathrm{correct}},\rho_{\mathrm{wrong}})=(0.40,0.30)$, and ablate where the contiguous block is hidden. Table~\ref{tab:mask_strategy_ablation} compares prefix, middle, random contiguous-span, and suffix masking with sparse interpolation disabled.

\begin{table}[htbp]
\centering
\caption{Mask-placement ablation with the default adaptive endpoints fixed and sparse interpolation disabled.}
\label{tab:mask_strategy_ablation}
\vspace{1mm}
\begin{tabular}{@{}lc@{}}
\toprule
\textbf{Mask placement} & \textbf{AIME 2024} \\
\midrule
Prefix mask & 56.1 \\
Middle mask & 57.8 \\
Random contiguous mask & 57.2 \\
\rowcolor{gray!20}
Suffix mask & \textbf{58.6} \\
\bottomrule
\end{tabular}
\end{table}

Suffix masking is the strongest placement in this sweep. This is consistent with the visible-context argument in \Cref{thm:method_suffix_masking_optimality}: early reference tokens often contain setup, decomposition, and high-level plan information, so preserving them gives the privileged scorer a coherent context while still hiding downstream execution details.

\subsection{Sparse Interpolation Strength and Budget}
\label{sec:interpolation_ablation}

Finally, we fix suffix masking and the adaptive endpoints to $(\rho_{\mathrm{correct}},\rho_{\mathrm{wrong}})=(0.40,0.30)$, then vary the interpolation strength $k_{\mathrm{base}}$ and the selected-token fraction $f_{\mathrm{top}}$. Table~\ref{tab:interpolation_ablation} reports the resulting AIME 2024 Avg@12 scores. The row $(k_{\mathrm{base}},f_{\mathrm{top}})=(0,0)$ disables interpolation, while $f_{\mathrm{top}}=1.00$ applies interpolation at every generated-token position.

\begin{table}[htbp]
\centering
\caption{Ablation of sparse energy interpolation with suffix masking and default adaptive endpoints fixed.}
\label{tab:interpolation_ablation}
\vspace{1mm}
\begin{tabular}{@{}ccc@{}}
\toprule
$\boldsymbol{k_{\mathrm{base}}}$ & $\boldsymbol{f_{\mathrm{top}}}$ & \textbf{AIME 2024} \\
\midrule
0.00 & 0.00 & 58.6 \\
0.01 & 1.00 & 58.6 \\
0.03 & 1.00 & 58.9 \\
0.05 & 1.00 & 57.8 \\
0.03 & 0.01 & 59.2 \\
\rowcolor{gray!20}
0.03 & 0.03 & \textbf{59.7} \\
0.03 & 0.05 & 58.9 \\
\bottomrule
\end{tabular}
\end{table}

Sparse interpolation improves over the masking-only setting, but applying it everywhere is less effective. The best setting uses a moderate interpolation strength and a small selected-position budget, supporting PAINT's design choice to calibrate only entropy-mismatch positions rather than all stylistic continuation positions.

\input{sections/theory}

%% file: sections/theory.tex
\section{Formal Justification for Adaptive Masking, Privileged Teaching, Energy Fusion, and Suffix Masking}
\label{sec:suffix_masking_theory}

\begin{proposition}[Autoregressive--Gibbs correspondence]
\label{prop:arm_gibbs_correspondence}
Fix a condition $x$, a finite vocabulary $\mathcal{V}$, and a finite horizon $T$; variable-length generation can be handled by including an EOS token and padding to the terminal horizon. Let $s_i=(x,a_{<i})$ be the autoregressive state at depth $i$. 

If a strictly positive autoregressive model $\pi$ assigns
\[
\pi(\tau\mid x)=\prod_{i=1}^T \pi(a_i\mid s_i)
\quad
\text{for } \tau=(a_1,\ldots,a_T),
\]
then the local costs
\[
c_\pi(s_i,a_i)=-\log \pi(a_i\mid s_i)
\]
define an additive energy $E_\pi(\tau)=\sum_{i=1}^T c_\pi(s_i,a_i)$ whose Gibbs distribution equals $\pi$:
\[
\frac{\exp(-E_\pi(\tau))}{\sum_{\tau'}\exp(-E_\pi(\tau'))}
=
\pi(\tau\mid x).
\]
Conversely, any additive energy $E(\tau)=\sum_{i=1}^T c(s_i,a_i)$ defines an autoregressive policy
\[
\pi_E(a_i\mid s_i)
=
\exp\!\left(-c(s_i,a_i)+V(s_{i+1})-V(s_i)\right),
\]
where
\[
V(s_i)=\log \sum_{a_i,\ldots,a_T}\exp\!\left(-\sum_{j=i}^T c(s_j,a_j)\right),
\qquad
V(s_{T+1})=0.
\]
\end{proposition}

\begin{proof}
For the first direction,
\[
\exp(-E_\pi(\tau))
=
\prod_{i=1}^T \pi(a_i\mid s_i)
=
\pi(\tau\mid x).
\]
Since $\pi$ is normalized over all length-$T$ trajectories, the Gibbs partition function is
\[
\sum_{\tau'}\exp(-E_\pi(\tau'))=\sum_{\tau'}\pi(\tau'\mid x)=1,
\]
so the Gibbs distribution is exactly $\pi$.

For the converse direction, define the state partition function
\[
Z(s_i)=\sum_{a_i,\ldots,a_T}\exp\!\left(-\sum_{j=i}^T c(s_j,a_j)\right),
\qquad
V(s_i)=\log Z(s_i).
\]
For a fixed action $a_i$,
\[
Z(s_i,a_i)
=
\sum_{a_{i+1},\ldots,a_T}
\exp\!\left(-c(s_i,a_i)-\sum_{j=i+1}^T c(s_j,a_j)\right)
=
\exp(-c(s_i,a_i))Z(s_{i+1}).
\]
Therefore the conditional probability of choosing $a_i$ under the Gibbs distribution is
\[
\pi_E(a_i\mid s_i)
=
\frac{Z(s_i,a_i)}{Z(s_i)}
=
\exp\!\left(-c(s_i,a_i)+V(s_{i+1})-V(s_i)\right).
\]
The conditionals sum to one by construction, so they define an autoregressive policy whose trajectory distribution is the Gibbs distribution.
\end{proof}

\begin{theorem}[Monotonicity of the optimal adaptive mask ratio]
\label{thm:appendix_adaptive_mask_monotonicity}
Let $\mathcal{X}$ be a finite set of valid continuations and let $\rho\in[0,1)$ denote a continuous hidden ratio for a contiguous masked block. Let
\[
q_\rho(\tau)
=
(1-\rho)q_0(\tau)+\rho q_1(\tau),
\qquad
\tau\in\mathcal{X},
\]
where $q_0,q_1\in\Delta(\mathcal{X})$, $q_0\neq q_1$, and $q_\rho(\tau)>0$ on the relevant support. For student parameters $\theta$, define
\[
M_2(q,\theta)
=
\mathbb{E}_{\tau\sim q}
\left[
\|\nabla_\theta \log p_\theta(\tau)\|^2
\right],
\qquad
v_0=M_2(q_0,\theta),
\qquad
\epsilon(\theta)=M_2(q_1,\theta)-M_2(q_0,\theta).
\]
Let
\[
H(\rho)
=
-\sum_{\tau\in\mathcal{X}} q_\rho(\tau)\log q_\rho(\tau),
\]
and consider the objective
\[
\mathcal{L}(\rho,\epsilon)
=
\beta H(\rho)-\lambda(v_0+\rho\epsilon),
\qquad
\beta,\lambda>0.
\]
Whenever an interior maximizer $\rho^\star(\epsilon)\in(0,1)$ exists, its differentiable branch is a strictly decreasing function of the alignment-error proxy:
\[
\frac{d\rho^\star}{d\epsilon}<0.
\]
\end{theorem}

\begin{proof}
By linearity of expectation under a convex combination of measures,
\[
\begin{aligned}
M_2(q_\rho,\theta)
&=
\mathbb{E}_{\tau\sim q_\rho}
\left[
\|\nabla_\theta \log p_\theta(\tau)\|^2
\right]
\\
&=
(1-\rho)M_2(q_0,\theta)+\rho M_2(q_1,\theta)
\\
&=
(1-\rho)v_0+\rho(v_0+\epsilon)
=
v_0+\rho\epsilon.
\end{aligned}
\]
Thus the objective can be written as
\[
\mathcal{L}(\rho,\epsilon)
=
\beta H(\rho)-\lambda(v_0+\rho\epsilon).
\]

Next, because $q_\rho$ is affine in $\rho$, its Shannon entropy is strictly concave along this path. Differentiating twice gives
\[
H''(\rho)
=
-\sum_{\tau\in\mathcal{X}}
\frac{(q_1(\tau)-q_0(\tau))^2}{q_\rho(\tau)}
<0,
\]
where strict negativity follows because $q_\rho(\tau)>0$ on the relevant support and $q_0\neq q_1$.

For an interior maximizer $\rho^\star(\epsilon)$, the first-order condition is
\[
\frac{\partial \mathcal{L}}{\partial \rho}
=
\beta H'(\rho^\star)-\lambda\epsilon
=
0.
\]
Define
\[
\mathcal{G}(\rho,\epsilon)
=
\beta H'(\rho)-\lambda\epsilon.
\]
Since
\[
\frac{\partial \mathcal{G}}{\partial \rho}
=
\beta H''(\rho^\star)\neq 0,
\]
the implicit function theorem applies and yields
\[
\frac{d\rho^\star}{d\epsilon}
=
-\frac{\partial \mathcal{G}/\partial \epsilon}
{\partial \mathcal{G}/\partial \rho}
=
-\frac{-\lambda}{\beta H''(\rho^\star)}
=
\frac{\lambda}{\beta H''(\rho^\star)}.
\]
Because $\lambda>0$, $\beta>0$, and $H''(\rho^\star)<0$, the ratio is strictly negative:
\[
\frac{d\rho^\star}{d\epsilon}<0.
\]
This proves the claim.
\end{proof}

\begin{assumption}[Energy-space sensor model]
\label{assump:energy_sensor_fusion}
Fix a rollout position $n$ and a prefix $\hat{y}_{<n}$. There is an ideal latent token-energy function $E_\star^n(v)$ for the desired next-token target. The student and privileged teacher energies are conditionally independent noisy observations of this latent energy:
\[
E_S^n(v)=E_\star^n(v)+\epsilon_S^n(v),
\qquad
E_T^n(v)=E_\star^n(v)+\epsilon_T^n(v),
\]
where $\epsilon_S^n(v)\sim\mathcal{N}(0,(\sigma_S^n)^2)$ and $\epsilon_T^n(v)\sim\mathcal{N}(0,(\sigma_T^n)^2)$.
\end{assumption}

\begin{proposition}[Precision-weighted geometric calibration]
\label{prop:precision_weighted_calibration}
Under \Cref{assump:energy_sensor_fusion}, assume a flat prior over $E_\star^n(v)$ and define $E_S^n(v)=-\log p_S^n(v)$ and $E_T^n(v)=-\log p_T^n(v)$ up to token-independent constants. Then the MAP estimate of the latent energy is
\[
E_{\mathrm{MAP}}^n(v)
=
k_n^\star E_S^n(v)
+
(1-k_n^\star)E_T^n(v),
\qquad
k_n^\star
=
\frac{(\sigma_T^n)^2}{(\sigma_T^n)^2+(\sigma_S^n)^2}.
\]
The normalized distribution induced by this energy is
\[
p_{\mathrm{MAP}}^n(v)
\propto
\bigl(p_S^n(v)\bigr)^{k_n^\star}
\bigl(p_T^n(v)\bigr)^{1-k_n^\star}.
\]
\end{proposition}

\begin{proof}
For a fixed token $v$, the conditional likelihood of the observed energies is
\[
p(E_S^n(v),E_T^n(v)\mid E_\star^n(v))
\propto
\exp\!\left(
-\frac{(E_S^n(v)-E_\star^n(v))^2}{2(\sigma_S^n)^2}
-\frac{(E_T^n(v)-E_\star^n(v))^2}{2(\sigma_T^n)^2}
\right).
\]
With a flat prior, maximizing the posterior is the same as maximizing the log-likelihood
\[
\mathcal{L}(E_\star^n(v))
=
-\frac{(E_S^n(v)-E_\star^n(v))^2}{2(\sigma_S^n)^2}
-\frac{(E_T^n(v)-E_\star^n(v))^2}{2(\sigma_T^n)^2}
+C.
\]
Setting the derivative to zero gives
\[
\frac{E_S^n(v)-E_\star^n(v)}{(\sigma_S^n)^2}
+
\frac{E_T^n(v)-E_\star^n(v)}{(\sigma_T^n)^2}
=0,
\]
hence
\[
E_{\mathrm{MAP}}^n(v)
=
\frac{\frac{1}{(\sigma_S^n)^2}E_S^n(v)+\frac{1}{(\sigma_T^n)^2}E_T^n(v)}
{\frac{1}{(\sigma_S^n)^2}+\frac{1}{(\sigma_T^n)^2}}
=
\frac{(\sigma_T^n)^2}{(\sigma_T^n)^2+(\sigma_S^n)^2}E_S^n(v)
+
\frac{(\sigma_S^n)^2}{(\sigma_T^n)^2+(\sigma_S^n)^2}E_T^n(v).
\]
This proves the precision-weighted energy formula. Finally, substituting $E_S^n(v)=-\log p_S^n(v)+c_S$ and $E_T^n(v)=-\log p_T^n(v)+c_T$ yields
\[
\exp(-E_{\mathrm{MAP}}^n(v))
\propto
\bigl(p_S^n(v)\bigr)^{k_n^\star}
\bigl(p_T^n(v)\bigr)^{1-k_n^\star},
\]
where the token-independent constants are absorbed into the normalizer. Normalizing over $v\in\mathcal{V}$ gives the stated distribution.
\end{proof}

\begin{proposition}[Entropy-ratio sparse selection]
\label{prop:entropy_ratio_sparse_selection}
Fix an interpolation strength $k_0\in(0,1)$. At a rollout position $n$, suppose the teacher and student define independent local Gaussian measures over a scalar energy coordinate,
\[
\pi_T^n(E)\propto
\exp\!\left(-\frac{(E-\mu_T^n)^2}{2(\sigma_T^n)^2}\right),
\qquad
\pi_S^n(E)\propto
\exp\!\left(-\frac{(E-\mu_S^n)^2}{2(\sigma_S^n)^2}\right),
\]
with precisions $\tau_T^n=1/(\sigma_T^n)^2$ and $\tau_S^n=1/(\sigma_S^n)^2$. Let
\[
\pi_k^n(E)
\propto
\bigl(\pi_T^n(E)\bigr)^{1-k}
\bigl(\pi_S^n(E)\bigr)^k
\]
be the geometric interpolation. If $H_S^n,H_T^n>0$ and the local energy variances obey a shared entropy--variance equation of state
\[
(\sigma_A^n)^2=c_n\bigl(H_A^n\bigr)^\gamma,
\qquad A\in\{S,T\},\quad c_n>0,\quad \gamma>0,
\]
then the signed differential-entropy reduction from applying $k_0$ is
\[
\Delta I_n(k_0)
=
h(\pi_T^n)-h(\pi_{k_0}^n)
=
\frac{1}{2}\log\!\left(1-k_0+k_0 R_n^\gamma\right),
\qquad
R_n=\frac{H_T^n}{H_S^n}.
\]
Moreover, $\Delta I_n(k_0)$ is strictly increasing in $R_n$; it is positive exactly when the interpolated Gaussian has lower entropy than the teacher Gaussian. Consequently, for any finite valid-position set $\mathcal{J}$ and any budget $S$, the problem
\[
\begin{aligned}
\max_{\mathbf{m}\in\{0,1\}^{|\mathcal{J}|}}
\quad&
\sum_{n\in\mathcal{J}}m_n\,\Delta I_n(k_0)
\\
\mathrm{s.t.}\quad&
\sum_{n\in\mathcal{J}}m_n=S
\end{aligned}
\]
is solved by selecting the $S$ largest entropy ratios $R_n$, with arbitrary tie-breaking at the cutoff.
\end{proposition}

\begin{proof}
Taking logarithms of the geometric interpolation gives
\[
\log \pi_k^n(E)
=
(1-k)\log \pi_T^n(E)+k\log \pi_S^n(E)+C.
\]
The coefficient of $E^2$ is therefore
\[
-\frac{1}{2}
\left((1-k)\tau_T^n+k\tau_S^n\right).
\]
Since a one-dimensional Gaussian with precision $\tau$ has quadratic coefficient $-\tau/2$, the interpolated precision is
\[
\tau_k^n
=
(1-k)\tau_T^n+k\tau_S^n .
\]
For a one-dimensional Gaussian, $h(\pi)=\frac{1}{2}\log(2\pi e)-\frac{1}{2}\log \tau$. Hence
\[
\Delta I_n(k)
=
h(\pi_T^n)-h(\pi_k^n)
=
\frac{1}{2}\log\frac{\tau_k^n}{\tau_T^n}
=
\frac{1}{2}\log\!\left(1-k+k\frac{\tau_S^n}{\tau_T^n}\right).
\]
The shared entropy--variance law gives
\[
\frac{\tau_S^n}{\tau_T^n}
=
\frac{(\sigma_T^n)^2}{(\sigma_S^n)^2}
=
\frac{c_n(H_T^n)^\gamma}{c_n(H_S^n)^\gamma}
=
R_n^\gamma,
\]
which proves the stated information-gain formula.

It remains to show that the nonlinear score preserves the ordering. Write $x=\log R_n$ and define
\[
g(x)=
\frac{1}{2}\log\!\left(1-k_0+k_0\exp(\gamma x)\right).
\]
Then
\[
g'(x)
=
\frac{1}{2}
\frac{k_0\gamma\exp(\gamma x)}
{1-k_0+k_0\exp(\gamma x)}
>0
\]
because $k_0\in(0,1)$ and $\gamma>0$. Thus $\Delta I_n(k_0)$, $\log R_n$, and $R_n$ induce the same ordering.

Finally, consider any feasible selector that includes a position $a$ but excludes a position $b$ with $\Delta I_b(k_0)>\Delta I_a(k_0)$. Swapping $a$ out and $b$ in strictly increases the objective while preserving the budget. Therefore no optimum can omit a larger score in favor of a smaller one, so every optimum selects the top $S$ scores; since the scores are order-equivalent to $R_n$, this is exactly top-$S$ entropy-ratio selection. Ties at the cutoff give equivalent optima.
\end{proof}

\begin{proposition}[Entropy smoothing]
\label{prop:mask_entropy_smoothing}
Let $y^\star=(y^\star_1,\ldots,y^\star_T)$, let $M_t=y^\star_{t:t+k-1}$, let $C_t=(y^\star_{<t},y^\star_{t+k:T})$, and let $\Xi_t^n=(x,\hat{y}_{<n},C_t)$. For
\[
q_{m,t}^n(a)=\Pr(A_n=a\mid \Xi_t^n,M_t=m),
\qquad
\bar{q}_t^n(a)=\sum_m \rho_t(m\mid\Xi_t^n)q_{m,t}^n(a),
\]
where $\rho_t(m\mid\Xi_t^n)=\Pr(M_t=m\mid\Xi_t^n)$,
\[
H(\bar{q}_t^n)
\ge
\mathbb{E}_{m\sim\rho_t(\cdot\mid\Xi_t^n)}
\left[H(q_{m,t}^n)\right].
\]
Equality holds if and only if $q_{m,t}^n$ is identical for all $m$ with positive posterior probability, up to events of zero probability.
\end{proposition}

\begin{proof}
Let $f(z)=-z\log z$, with the usual convention $0\log 0=0$. The function $f$ is concave on $[0,1]$ and strictly concave on $(0,1)$. Since
\[
\bar{q}_t^n(a)=\sum_m \rho_t(m\mid\Xi_t^n)q_{m,t}^n(a),
\]
Jensen's inequality gives, for each token $a$,
\[
f(\bar{q}_t^n(a))
\ge
\sum_m \rho_t(m\mid\Xi_t^n)f(q_{m,t}^n(a)).
\]
Summing over $a\in\mathcal{V}$ yields the entropy inequality. The equality condition follows from the strict concavity of $f$ on the positive-probability coordinates: equality can hold only when the token probabilities agree across all posterior-supported hidden blocks.
\end{proof}

\begin{lemma}[Forward-KL decomposition]
\label{lem:mask_kl_decomposition}
Under the setup of \Cref{prop:mask_entropy_smoothing}, let $p_S^n(a)=q_{\mathrm{base}}(a \mid x,\hat{y}_{<n})$ be the student distribution, with support containing the support of every $q_{m,t}^n$. Then
\[
\mathbb{E}_{m\sim\rho_t(\cdot\mid\Xi_t^n)}
\left[
D_{\mathrm{KL}}(q_{m,t}^n\,\|\,p_S^n)
\right]
=
D_{\mathrm{KL}}(\bar{q}_t^n\,\|\,p_S^n)
+
I(A_n;M_t\mid \Xi_t^n).
\]
\end{lemma}

\begin{proof}
Expanding the left-hand side,
\[
\begin{aligned}
&\sum_m \rho_t(m\mid\Xi_t^n)
\sum_a q_{m,t}^n(a)\log\frac{q_{m,t}^n(a)}{p_S^n(a)}
\\
&=
\sum_m \rho_t(m\mid\Xi_t^n)
\sum_a q_{m,t}^n(a)\log\frac{q_{m,t}^n(a)}{\bar{q}_t^n(a)}
+
\sum_m \rho_t(m\mid\Xi_t^n)
\sum_a q_{m,t}^n(a)\log\frac{\bar{q}_t^n(a)}{p_S^n(a)}.
\end{aligned}
\]
The first term is exactly $I(A_n;M_t\mid\Xi_t^n)$ under the joint distribution
\[
\Pr(M_t=m,A_n=a\mid\Xi_t^n)
=
\rho_t(m\mid\Xi_t^n)q_{m,t}^n(a).
\]
For the second term, use $\sum_m\rho_t(m\mid\Xi_t^n)q_{m,t}^n(a)=\bar{q}_t^n(a)$ to obtain
\[
\sum_a \bar{q}_t^n(a)\log\frac{\bar{q}_t^n(a)}{p_S^n(a)}
=
D_{\mathrm{KL}}(\bar{q}_t^n\,\|\,p_S^n).
\]
This proves the identity.
\end{proof}

\begin{proposition}[Rao--Blackwellized gradient variance]
\label{prop:mask_variance}
Under the setup of \Cref{prop:mask_entropy_smoothing}, assume the target distributions are treated as fixed quantities, as in the frozen-teacher implementation. Define
\[
G_{m,t}^n(\theta)
=
\nabla_\theta D_{\mathrm{KL}}(q_{m,t}^n\,\|\,p_S^n),
\qquad
\bar{G}_t^n(\theta)
=
\nabla_\theta D_{\mathrm{KL}}(\bar{q}_t^n\,\|\,p_S^n).
\]
Then
\[
\mathbb{E}_{m\sim\rho_t(\cdot\mid\Xi_t^n)}
\left[G_{m,t}^n(\theta)\right]
=
\bar{G}_t^n(\theta).
\]
Moreover, over the joint randomness of examples, rollout prefixes, positions, and hidden blocks,
\[
\operatorname{Cov}(G_{M_t,t}^n)
=
\operatorname{Cov}(\bar{G}_t^n)
+
\mathbb{E}\!\left[
\operatorname{Cov}(G_{M_t,t}^n\mid \Xi_t^n)
\right],
\]
so the covariance difference is positive semidefinite.
\end{proposition}

\begin{proof}
Because the target is held fixed when differentiating,
\[
G_{m,t}^n(\theta)
=
-\sum_a q_{m,t}^n(a)\nabla_\theta\log p_S^n(a).
\]
Taking the posterior expectation over $m$ gives
\[
\mathbb{E}_m[G_{m,t}^n(\theta)\mid\Xi_t^n]
=
-\sum_a \bar{q}_t^n(a)\nabla_\theta\log p_S^n(a)
=
\bar{G}_t^n(\theta).
\]
Conditioning on the sampled prefixes and treating rollout sampling as fixed for the supervised distillation update, the covariance identity is the law of total covariance applied to the random vector $G_{M_t,t}^n$ with conditioning variable $\Xi_t^n$:
\[
\operatorname{Cov}(G_{M_t,t}^n)
=
\operatorname{Cov}(\mathbb{E}[G_{M_t,t}^n\mid\Xi_t^n])
+
\mathbb{E}[\operatorname{Cov}(G_{M_t,t}^n\mid\Xi_t^n)].
\]
Substituting $\mathbb{E}[G_{M_t,t}^n\mid\Xi_t^n]=\bar{G}_t^n$ gives the stated result. Conditional covariance matrices are positive semidefinite, so the final term is positive semidefinite.
\end{proof}

\begin{lemma}[Visible-context duality]
\label{lem:visible_context_duality}
Let $M_t=y^\star_{t:t+k-1}$ and $C_t=(y^\star_{<t},y^\star_{t+k:T})$ for a fixed masking budget $k$, and define
\[
\Gamma(t)
=
\mathbb{E}_{(x,y^\star),\hat{y},n}
\left[
I(A_n;M_t\mid x,\hat{y}_{<n},C_t)
\right],
\]
where $n$ is sampled uniformly from the rollout positions. Then
\[
\Gamma(t)
=
K
-
\mathbb{E}_{(x,y^\star),\hat{y},n}
\left[
I(A_n;C_t\mid x,\hat{y}_{<n})
\right],
\]
where
\[
K=
\mathbb{E}_{(x,y^\star),\hat{y},n}
\left[
I(A_n;y^\star\mid x,\hat{y}_{<n})
\right]
\]
does not depend on $t$.
\end{lemma}

\begin{proof}
For a fixed $t$ and $k$, the pair $(C_t,M_t)$ is a deterministic partition of the same reference solution $y^\star$. By the chain rule for mutual information,
\[
I(A_n;C_t,M_t\mid x,\hat{y}_{<n})
=
I(A_n;C_t\mid x,\hat{y}_{<n})
+
I(A_n;M_t\mid x,\hat{y}_{<n},C_t).
\]
Since $(C_t,M_t)$ and $y^\star$ carry the same information,
\[
I(A_n;C_t,M_t\mid x,\hat{y}_{<n})
=
I(A_n;y^\star\mid x,\hat{y}_{<n}).
\]
Rearranging and taking expectations gives the result. The first term on the right is independent of the mask position because it conditions on the complete reference solution.
\end{proof}

\begin{assumption}[Causal information asymmetry]
\label{assump:causal_info_asymmetry}
For a fixed masking budget $k$, later mask positions leave a weakly more informative visible context for predicting the teacher's next action: for every $t_1<t_2$,
\[
\mathbb{E}
\left[
I(A_n;C_{t_1}\mid x,\hat{y}_{<n})
\right]
\le
\mathbb{E}
\left[
I(A_n;C_{t_2}\mid x,\hat{y}_{<n})
\right]
\]
holds.
\end{assumption}

\begin{theorem}[Optimality of suffix masking]
\label{thm:suffix_masking_optimality}
Under \Cref{lem:visible_context_duality,assump:causal_info_asymmetry}, $\Gamma(t)$ is non-increasing in $t$, and the minimum expected information loss is attained at the largest feasible mask position,
\[
t^\star=T-k+1,
\]
which masks the suffix of length $k$.
\end{theorem}

\begin{proof}
Let $t_1<t_2$. By \Cref{lem:visible_context_duality},
\[
\Gamma(t_1)-\Gamma(t_2)
=
\mathbb{E}\!\left[I(A_n;C_{t_2}\mid x,\hat{y}_{<n})\right]
-
\mathbb{E}\!\left[I(A_n;C_{t_1}\mid x,\hat{y}_{<n})\right].
\]
The assumed causal information asymmetry makes the right-hand side nonnegative, so $\Gamma(t_1)\ge \Gamma(t_2)$. Therefore $\Gamma(t)$ is non-increasing in $t$, and its minimum over $t\in\{1,\ldots,T-k+1\}$ is achieved at $t=T-k+1$.
\end{proof}